\pdfoutput=1
\documentclass[11pt,a4paper]{article}
\usepackage[hyperref]{acl2021}
\usepackage{graphicx}
\usepackage{times}
\usepackage{latexsym}
\usepackage{tabu}
\usepackage{tabularx}
\usepackage{multirow}
\usepackage{array}
\usepackage{amsfonts}
\usepackage{bm}
\usepackage{booktabs}
\usepackage{mathptmx}
\usepackage{amsmath}
\usepackage{amssymb}
\usepackage{amsthm}
\usepackage{mathtools}
\usepackage{tikz}
\usepackage{pgfplots}
\newtheorem{prop}{Proposition}
\usepackage{algorithm,algorithmic}

\usepackage[above,below]{placeins}

\usepackage{microtype}

\aclfinalcopy 
\def\aclpaperid{1415} 

\title{Unified Interpretation of Softmax Cross-Entropy and Negative Sampling: \\With Case Study for Knowledge Graph Embedding}

\author{Hidetaka Kamigaito \\
  Tokyo Institute of Technology \\
  \texttt{kamigaito@lr.pi.titech.ac.jp} \\\And
  Katsuhiko Hayashi \\
  Gunma University, RIKEN AIP \\\
  \texttt{khayashi0201@gmail.com} \\}

\date{}

\begin{document}
\maketitle
\begin{abstract}
In knowledge graph embedding, the theoretical relationship between the softmax cross-entropy and negative sampling loss functions has not been investigated. This makes it difficult to fairly compare the results of the two different loss functions. We attempted to solve this problem by using the Bregman divergence to provide a unified interpretation of the softmax cross-entropy and negative sampling loss functions. Under this interpretation, we can derive theoretical findings for fair comparison. Experimental results on the FB15k-237 and WN18RR datasets show that the theoretical findings are valid in practical settings.
\end{abstract}

\section{Introduction}
Negative Sampling (NS) \cite{DBLP:journals/corr/MikolovSCCD13} is an approximation of softmax cross-entropy (SCE). Due to its efficiency in computation cost, NS is now a fundamental loss function for various Natural Language Processing (NLP) tasks such as used in word embedding \cite{DBLP:journals/corr/MikolovSCCD13}, language modeling \cite{melamud-etal-2017-simple}, contextualized embedding \cite{clark-etal-2020-pre,Clark2020ELECTRA:}, and knowledge graph embedding (KGE) \cite{DBLP:conf/icml/TrouillonWRGB16}. Specifically, recent KGE models commonly use NS for training. Considering the current usages of NS, we investigated the characteristics of NS by mainly focusing on KGE from theoretical and empirical aspects.

First, we introduce the task description of KGE. A knowledge graph is a graph that describes the relationships between entities. It is an indispensable resource for knowledge-intensive NLP applications such as dialogue \cite{moon-etal-2019-opendialkg} and question-answering \cite{10.1145/3038912.3052675} systems. However, to create a knowledge graph, it is necessary to consider a large number of entity combinations and their relationships, making it difficult to construct a complete graph manually. Therefore, the prediction of links between entities is an important task.

Currently, missing relational links between entities are predicted using a scoring method based on KGE \cite{10.5555/2900423.2900470}. With this method, a score for each link is computed on vector space representations of embedded entities and relations. We can train these representations through various loss functions. The SCE \cite{kadlec-etal-2017-knowledge} and NS \cite{DBLP:conf/icml/TrouillonWRGB16} loss functions are commonly used for this purpose.

Several studies \cite{Ruffinelli2020You,ali2020pykeen} have shown that link-prediction performance can be significantly improved by choosing the appropriate combination of loss functions and scoring methods. However, the relationship between the SCE and NS loss functions has not been investigated in KGE. Without a basis for understanding the relationships among different loss functions, it is difficult to make a fair comparison between the SCE and NS results.

We attempted to solve this problem by using the Bregman divergence \cite{BREGMAN1967200} to provide a unified interpretation of the SCE and NS loss functions. Under this interpretation, we can understand the relationships between SCE and NS in terms of the model’s predicted distribution at the optimal solution, which we called the \textbf{objective distribution}. By deriving the objective distribution for a loss function, we can analyze different loss functions, the objective distributions of which are identical under certain conditions, from a unified viewpoint.

We summarize our theoretical findings not restricted to KGE as follows:
\begin{itemize}
\item The objective distribution of NS with uniform noise (NS w/ Uni) is equivalent to that of SCE.
\item The objective distribution of self-adversarial negative sampling (SANS) \cite{DBLP:journals/corr/abs-1902-10197} is quite similar to SCE with label smoothing (SCE w/ LS) \cite{label1}.
\item NS with frequency-based noise (NS w/ Freq) in word2vec\footnote{The word2vec uses unigram distribution as the frequency-based noise.} has a smoothing effect on the objective distribution.
\item SCE has a property wherein it more strongly fits a model to the training data than NS.
\end{itemize}

To check the validity of the theoretical findings in practical settings, we conducted experiments on the FB15k-237~\cite{toutanova-chen-2015-observed} and WN18RR~\cite{dettmers2018conve} datasets. The experimental results indicate that
\begin{itemize}
\item The relationship between SCE and SCE w/ LS is also similar to that between NS and SANS in practical settings.
\item NS is prone to underfitting because it weakly fits a model to the training data compared with SCE.
\item SCE causes underfitting of KGE models when their score function has a bound.
\item Both SANS and SCE w/ LS perform well as pre-training methods.
\end{itemize}

The structure of this paper is as follows: Sec.~2 introduces SCE and Bregman divergence; Sec.~3 induces the objective distributions for NS; Sec.~4 analyzes the relationships between SCE and NS loss functions; Sec.~5 summarizes and discusses our theoretical findings; Sec.~6 discusses empirically investigating the validity of the theoretical findings in practical settings; Sec.~7 explains the differences between this paper and related work; and Sec.~8 summarizes our contributions. Our code will be available at \url{https://github.com/kamigaito/acl2021kge}

\section{Softmax Cross Entropy and Bregman Divergence}
\label{sec:bd}
\subsection{SCE in KGE}
We denote a link representing a relationship $r_{k}$ between entities $e_{i}$ and $e_{j}$ in a knowledge graph as $(e_{i},r_{k},e_{j})$. In predicting the links from given queries $(e_{i},r_{k},?)$ and $(?,r_{k},e_{j})$, the model must predict entities corresponding to each $?$ in the queries. We denote such a query as $x$ and the entity to be predicted as $y$. By using the softmax function, the probability $p_{\theta}(y|x)$ that $y$ is predicted from $x$ with the model parameter $\theta$ given a score function $f_{\theta}(x,y)$ is expressed as follows:
\begin{equation}
 p_{\theta}(y|x)=\frac{
\exp{(f_{\theta}(x,y))}
}{
\sum_{y'\in Y}{
\exp{(f_{\theta}(x,y'))}
}
},
\label{eq:softmax}
\end{equation}

where $Y$ is the set of all predictable entities. We further denote the pair of an input $x$ and its label $y$ as $(x,y)$. Let $D=\{(x_{1},y_{1}),\cdots,(x_{|D|},y_{|D|})\}$ be observed data that obey a distribution $p_{d}(x,y)$.

\subsection{Bregman Divergence}
Next, we introduce the Bregman divergence. Let $\Psi(z)$ be a differentiable function; the Bregman divergence between two distributions $f$ and $g$ is defined as follows:
\begin{equation}
 d_{\Psi(z)}(f,g) = \Psi(f) - \Psi(g) - \nabla\Psi(g)^{T}(f-g).
\end{equation}

We can express various divergences by changing $\Psi(z)$. To take into account the divergence on the entire observed data, we consider the expectation of $d_{\Psi}(f,g)$: $B_{\Psi(z)}(f, g) = \sum_{x,y} d_{\Psi(z)}(f(y|x), g(y|x))p_{d}(x,y)$. To investigate the relationship between a loss function and learned distribution of a model at an optimal solution of the loss function, we need to focus on the minimization of $B_{\Psi(z)}$. \newcite{10.5555/3020548.3020582} showed that $B_{\Psi(z)}(f, g) = 0$ means that $f$ equals $g$ almost everywhere when $\Psi(z)$ is a differentiable strictly convex function in its domain. Note that all $\Psi(z)$ in this paper satisfy this condition. Accordingly, by fixing $f$, minimization of $B_{\Psi(z)}(f, g)$ with respect to $g$ is equivalent to minimization of
\begin{equation}
\scalebox{0.89}{$\begin{aligned}
 &\tilde{B}_{\Psi(z)}(f,g)\\
 =&\sum_{x,y} \left[-\Psi(g) + \nabla\Psi(g)^{T}g - \nabla\Psi(g)^{T}f \right]p_{d}(x,y)
\end{aligned}$}\label{eq:bd:reform}
\end{equation}
We use $\tilde{B}_{\Psi}(f,g)$ to reveal a learned distribution of a model at optimal solutions for the SCE and NS loss functions.

\subsection{Derivation of SCE}
For the latter explanations, we first derive the SCE loss function from Eq. (\ref{eq:bd:reform}). We denote a probability for a label $y$ as $p(y)$, vector for all $y$ as $\mathbf{y}$, vector of probabilities for $\mathbf{y}$ as $p(\mathbf{y})$, and dimension size of $\mathbf{z}$ as $len(\mathbf{z})$. In Eq. (\ref{eq:bd:reform}), by setting $f$ as $p_{d}(\mathbf{y}|x)$ and $g$ as $p_{\theta}(\mathbf{y}|x)$ with $\Psi(\mathbf{z})=\sum_{i=1}^{len(\mathbf{z})}z_{i}\log{z_{i}}$ \cite{JMLR:v6:banerjee05b}, we can derive the SCE loss function as follows:
\begin{align}
\!\!&\tilde{B}_{\Psi(\mathbf{z})}(p_{d}(\mathbf{y}|x),p_{\theta}(\mathbf{y}|x))\nonumber\\
\!\!=& - \sum_{x,y} \left[ \sum_{i=1}^{|Y|}p_{d}(y_{i}|x)\log{p_{\theta}(y_i|x)}\right]p_{d}(x,y)\label{eq:ce:middle}\\
\!\!=& - \frac{1}{|D|}\sum_{(x,y) \in D}\log{p_{\theta}(y|x)}.\label{eq:ce:final}
\end{align}
This derivation indicates that $p_{\theta}(y|x)$ converges to the observed distribution $p_{d}(y|x)$ through minimizing $B_{\Psi(\mathbf{z})}(p_{d}(\mathbf{y}|x),p_{\theta}(\mathbf{y}|x))$ in the SCE loss function. We call the distribution of $p_{\theta}(\mathbf{y}|x)$ when $B_{\Psi(\mathbf{z})}$ equals zero an \textbf{objective distribution}.

\section{Objective Distribution for Negative Sampling Loss}
\label{sec:ns}
We begin by providing a definition of NS and its relationship to the Bregman divergence, following the induction of noise contrastive estimation~(NCE) from the Bregman divergence that was established by~\citet{10.5555/3020548.3020582}. We denote $p_{n}(y|x)$ to be a known non-zero noise distribution for $y$ of a given $x$. Given $\nu$ noise samples from $p_{n}(y|x)$ for each $(x,y)\in D$, NS estimates the model parameter $\theta$ for a distribution $G(y|x;\theta) = \exp(-f_{\theta}(x,y))$.

By assigning to each $(x,y)$ a binary class label $C$: $C = 1$ if $(x,y)$ is drawn from observed data $D$ following a distribution $p_{d}(x,y)$ and $C = 0$ if $(x,y)$ is drawn from a noise distribution $p_{n}(y|x)$, we can model the posterior probabilities for the classes as follows:
\vspace{-2mm}
\begin{equation}
\begin{aligned}
 p(C=1,y|x;\theta) &= \frac{1}{1 \!+\! \exp(-f_{\theta}(x,y))} \\
 &=\frac{1}{1+G(y|x;\theta)}, 
 \\
 p(C=0,y|x;\theta) &= 1\!-\!p(C=1,y|x;\theta) \\
 &=\frac{G(y|x;\theta)}{1+G(y|x;\theta)}. 
\end{aligned}\nonumber
\end{equation}
The objective function $\ell^{NS}(\theta)$ of NS is defined as follows:
\begin{align}
 \ell^{NS}(\theta)= &-\frac{1}{|D|}\sum_{(x,y) \in D} \Bigl[\log(P(C=1,y|x;\theta)) \bigr. \nonumber\\
 &\bigl.+\sum_{i=1,y_{i}\sim p_n}^{\nu}\log(P(C=0,y_{i}|x;\theta))\Bigr].\label{eq:ns:loss}
\end{align}
By using the Bregman divergence, we can induce the following propositions for $\ell^{NS}(\theta)$.
\begin{prop}
\label{prop:ns:psi}
$\ell^{NS}(\theta)$ can be induced from Eq. (\ref{eq:bd:reform}) by setting $\Psi(z)$ as:
\begin{equation}
\Psi(z) = z\log(z)-(1+z)\log(1+z).
\end{equation}
\end{prop}
\begin{prop}
\label{prop:ns:g}
When $\ell^{NS}(\theta)$ is minimized, the following equation is satisfied:
\begin{equation}
 G(y|x;\theta) = \frac{\nu p_{n}(y|x)}{p_{d}(y|x)}.
\label{eq:ns:obj-true}
\end{equation}
\end{prop}
\begin{prop}
\label{prop:ns:obj}
The objective distribution of $P_{\theta}(y|x)$ for $\ell^{NS}(\theta)$ is
\begin{equation}
\frac{p_{d}(y|x)}{p_{n}(y|x)\sum\limits_{y_{i}\in Y}\frac{p_{d}(y_i|x)}{p_{n}(y_i|x)}}.
\label{eq:ns:softmax}
\end{equation}
\end{prop}
\begin{proof}
We give the proof of Props. \ref{prop:ns:psi}, \ref{prop:ns:g}, and \ref{prop:ns:obj} in Appendix \ref{breg:ns:details} of the supplemental material.
\end{proof}

We can also investigate the validity of Props. \ref{prop:ns:psi}, \ref{prop:ns:g}, and \ref{prop:ns:obj} by comparing them with the previously reported result. For this purpose, we prove the following proposition:
\begin{prop}
When Eq. (\ref{eq:ns:obj-true}) satisfies $\nu=1$ and $p_{n}(y|x)=p_{d}(y)$, $f_{\theta}(x,y)$ equals point-wise mutual information (PMI).
\end{prop}
\begin{proof}
This is described in Appendix \ref{app:proof:pmi} of the supplemental material.
\end{proof}

This observation is consistent with that by \citet{10.5555/2969033.2969070}. The differences between their representation and ours are as follows. (1) Our noise distribution is general in the sense that its definition is not restricted to a unigram distribution; (2) we mainly discuss $p_{\theta}(y|x)$ not $f_{\theta}(x,y)$; and (3) we can compare NS- and SCE-based loss functions through the Bregman divergence.

\subsection{Various Noise Distributions}
\label{sec:obj:dist:noise}
Different from the objective distribution of SCE, Eq. (\ref{eq:ns:softmax}) is affected by the type of noise distribution $p_{n}(y|x)$. To investigate the actual objective distribution for $\ell^{NS}(\theta)$, we need to consider separate cases for each type of noise distribution. In this subsection, we further analyze Eq. (\ref{eq:ns:softmax}) for each separate case.

\subsubsection{NS with Uniform Noise}
First, we investigated the case of a uniform distribution because it is one of the most common noise distributions for $\ell^{NS}(\theta)$ in the KGE task. From Eq. (\ref{eq:ns:softmax}), we can induce the following property.
\begin{prop}
When $p_{n}(y|x)$ is a uniform distribution, Eq. (\ref{eq:ns:softmax}) equals $p_{d}(y|x)$.
\end{prop}
\begin{proof}
This is described in Appendix \ref{app:proof:uninoise} of the supplemental material.
\end{proof}
\citet{nce-ns} indicated that NS is equal to NCE when $\nu=|Y|$ and $P_{n}(y|x)$ is uniform. However, as we showed, in terms of the objective distribution, the value of $\nu$ is not related to the objective distribution because Eq. (\ref{eq:ns:softmax}) is independent of $\nu$.

\subsubsection{NS with Frequency-based Noise}
\label{subsec:ns:freq}
In the original setting of NS~\cite{DBLP:journals/corr/MikolovSCCD13}, the authors chose as $p_{n}(y|x)$ a unigram distribution of $y$, which is independent of $x$. Such a frequency-based distribution is calculated in terms of frequencies on a corpus and independent of the model parameter $\theta$. Since in this case, different from the case of a uniform distribution, $p_{n}(y|x)$ remains on the right side of Eq.~(\ref{eq:ns:softmax}), $p_{\theta}(y|x)$ decreases when $p_{n}(y|x)$ increases. Thus, we can interpret frequency-based noise as a type of smoothing for $p_{d}(y|x)$. The smoothing of NS w/ Freq decreases the importance of high-frequency labels in the training data for learning more general vector representations, which can be used for various tasks as pre-trained vectors. Since we can expect pre-trained vectors to work as a prior \cite{10.5555/1756006.1756025} that prevents models from overfitting, we tried to use NS w/ Freq for pre-training KGE models in our experiments.

\subsubsection{Self-Adversarial NS}
\newcite{DBLP:journals/corr/abs-1902-10197} recently proposed SANS, which uses $p_{\theta}(y|x)$ for generating negative samples. By replacing $p_{n}(y|x)$ with $p_{\theta}(y|x)$, the objective distribution when using SANS is as follows:
\begin{align}
 p_{\theta}(y|x) = \frac{p_{d}(y|x)}{p_{\hat{\theta}}(y|x)\sum\limits_{y_{i}\in Y}\frac{p_{d}(y_i|x)}{p_{\hat{\theta}}(y_i|x)}}\label{eq:sa:ns:softmax},
\end{align}
where $\hat{\theta}$ is a parameter set updated in the previous iteration. Because both the left and right sides of Eq. (\ref{eq:sa:ns:softmax}) include $p_{\theta}(y|x)$, we cannot obtain an analytical solution of $p_{\theta}(y|x)$ from this equation. However, we can consider special cases of $p_{\theta}(y|x)$ to gain an understanding of Eq. (\ref{eq:sa:ns:softmax}). At the beginning of the training, $p_{\theta}(y|x)$ follows a discrete uniform distribution $u\{1,|Y|\}$ because $\theta$ is randomly initialized. In this situation, when we set $p_{\hat{\theta}}(y|x)$ in Eq. (\ref{eq:sa:ns:softmax}) to a discrete uniform distribution $u\{1,|Y|\}$, $p_{\theta}(y|x)$ becomes
\begin{align}
p_{\theta}(y|x) = p_{d}(y|x). \label{eq:sa:obj:1}
\end{align}
Next, when we set $p_{\hat{\theta}}(y|x)$ in Eq. (\ref{eq:sa:ns:softmax}) as $p_{d}(y|x)$, $p_{\theta}(y|x)$ becomes
\begin{align}
p_{\theta}(y|x) = u\{1,|Y|\}. \label{eq:sa:obj:2}
\end{align}

In actual mini-batch training, $\theta$ is iteratively updated for every batch of data. Because $p_{\theta}(y|x)$ converges to $u\{1,|Y|\}$ when $p_{\hat{\theta}}(y|x)$ is close to $p_{d}(y|x)$ and $p_{\theta}(y|x)$ converges to $p_{d}(y|x)$ when $p_{\hat{\theta}}(y|x)$ is close to $u\{1,|Y|\}$, we can approximately regard the objective distribution of SANS as a mixture of $p_{d}$ and $u\{1,|Y|\}$. Thus, we can represent the objective distribution of $p_{\theta}(y|x)$ as
\begin{align}
p_{\theta}(y|x) \approx (1-\lambda) p_{d}(y|x) + \lambda u\{1,|Y|\}
\end{align}
where $\lambda$ is a hyper-parameter to determine whether $p_{\theta}(y|x)$ is close to $p_{d}(y|x)$ or $u\{1,|Y|\}$. Assuming that $p_{\theta}(y|x)$ starts from $u\{1,|Y|\}$, $\lambda$ should start from $0$ and gradually increase through training. Note that $\lambda$ corresponds to a temperature $\alpha$ for $p_{\hat{\theta}}(y|x)$ in SANS, defined as
\begin{align}
 p_{\hat{\theta}}(y|x) = \frac{\exp(\alpha f_{\theta}(x,y))}{\sum_{y'\in Y} \exp(\alpha f_{\theta}(x,y'))},
 \label{eq:sans:temp}
\end{align}
where $\alpha$ also adjusts $p_{\hat{\theta}}(y|x)$ to be close to $p_{d}(y|x)$ or $u\{1,|Y|\}$.

\begin{table*}[t]
 \centering
 \small
 \begin{tabular}{lccp{4.5cm}}
 \toprule
 Loss & Objective Distribution & $\Psi(z)$ or $\Psi(\mathbf{z})$ & Remarks \\
 \cmidrule{1-4}
 NS w/ Uni & $p_{d}(y|x)$ & $\Psi(z)=z\log(z)-(1+z)\log(1+z)$ & \\
 NS w/ Freq & $T^{-1}_{x,y}p_{d}(y|x)$ & $\Psi(z)=z\log(z)-(1+z)\log(1+z)$ & $T_{x,y}=p_{n}(y|x)\sum\limits_{y_{i}\in Y}\frac{p_{d}(y_i|x)}{p_{n}(y_i|x)}$\\
 SANS & $(1-\lambda) p_{d}(y|x) + \lambda u\{1,|Y|\}$ & $\Psi(z)=z\log(z)-(1+z)\log(1+z)$ & Approximately derived. $\lambda$ increases from zero in training. \\
 \cmidrule{1-4}
 SCE & $p_{d}(y|x)$ & $\Psi(\mathbf{z})=\sum_{i=1}^{len(\mathbf{z})}z_{i}\log{z_{i}}$ & \\
 SCE w/ BC & $T^{-1}_{x,y}p_{d}(y|x)$ & $\Psi(\mathbf{z})=\sum_{i=1}^{len(\mathbf{z})}z_{i}\log{z_{i}}$ & $T_{x,y}=p_{n}(y|x)\sum\limits_{y_{i}\in Y}\frac{p_{d}(y_i|x)}{p_{n}(y_i|x)}$ \\
 SCE w/ LS & $(1-\lambda)p_{d}(y|x) + \lambda u\{1,|Y|\}$ & $\Psi(\mathbf{z})=\sum_{i=1}^{len(\mathbf{z})}z_{i}\log{z_{i}}$ & $\lambda$ is fixed.\\
 \bottomrule
 \end{tabular}
 \caption{Summary of our theoretical findings. w/ Uni denotes with uniform noise, w/ Freq denotes with frequency-based noise, w/ BC denotes with backward correction, and w/ LS denotes with label smoothing.}
 \label{tab:overview:loss}
\end{table*}

\section{Theoretical Relationships among Loss Functions}
\label{sec:threl}

\subsection{Corresponding SCE form to NS with Frequency-based Noise}
We induce a corresponding cross entropy loss from the objective distribution for NS with frequency-based noise. We set $T_{x,y} = p_{n}(y|x)\sum\limits_{y_{i}\in Y}\frac{p_{d}(y_i|x)}{p_{n}(y_i|x)}$, $q(y|x)=T_{x,y}^{-1}p_{d}(y|x)$, and $\Psi(\mathbf{z})=\sum_{i=1}^{len(\mathbf{z})}z_{i}\log{z_{i}}$. Under these conditions, following induction from Eq. (\ref{eq:ce:middle}) to Eq. (\ref{eq:ce:final}), we can reformulate $\tilde{B}_{\Psi(\mathbf{z})}(q(\mathbf{y}|x),p(\mathbf{y}|x))$ as follows:
\begin{align}
 &\tilde{B}_{\Psi(\mathbf{z})}(q(\mathbf{y}|x),p_{\theta}(\mathbf{y}|x))\nonumber\\
 =&-\sum_{x,y}\left[\sum_{i=1}^{|Y|} T_{x,y}^{-1}p_{d}(y_{i}|x)\log{p_{\theta}(y_{i}|x)}\right]p_{d}(x,y)\nonumber\\
 =&-\frac{1}{|D|}\sum_{(x,y)\in D} T_{x,y}^{-1}\log{p_{\theta}(y|x)}.\label{eq:ce:bc}
\end{align}
Except that $T_{x,y}$ is conditioned by $x$ and not normalized for $y$, we can interpret this loss function as SCE with backward correction (SCE w/ BC) \citep{patrini2016making}. Taking into account that backward correction can be a smoothing method for predicting labels \cite{pmlr-v119-lukasik20a}, this relationship supports the theoretical finding that NS can adopt a smoothing to the objective distribution.

Because the frequency-based noise is used in word2vec as unigram noise, we specifically consider the case in which $p_{n}(y|x)$ is set to unigram noise. In this case, we can set $p_{n}(y|x)=p_{d}(y)$. Since relation tuples do not appear twice in a knowledge graph, we can assume that $p_{d}(x,y)$ is uniform. Accordingly, we can change $T_{x,y}^{-1}$ to $\frac{1}{ p_{d}(y)\sum\limits_{y_{i}\in Y}\frac{p_{d}(y_i|x)}{p_{d}(y_i)}} = \frac{1}{ p_{d}(y)\sum\limits_{y_{i}\in Y}\frac{p_{d}(y_i,x)}{p_{d}(y_i)p_{d}(x)}}=\frac{p_{d}(x)}{p_{d}(y)C}$, where $C$ is a constant value, and we can reformulate Eq. (\ref{eq:ce:bc}) as follows:
\begin{align}
&-\frac{1}{|D|}\sum_{(x,y)\in D} \frac{p_{d}(x)}{p_{d}(y)C}\log{p_{\theta}(y|x)}\nonumber\\
\propto &-\frac{1}{|D|}\sum_{(x,y)\in D} \frac{\#x}{\#y}\log{p_{\theta}(y|x)},\label{eq:sce:unigram}
\end{align}
where $\#x$ and $\#y$ respectively represent frequencies for $x$ and $y$ in the training data. We use Eq. (\ref{eq:sce:unigram}) to pre-train models for SCE-based loss functions.

\subsection{Corresponding SCE form to SANS}
We induce a corresponding cross entropy loss from the objective distribution for SANS by setting $q(y|x)=(1-\lambda) p_{d}(y|x) + \lambda u\{1,|Y|\}$ and $\Psi(\mathbf{z})=\sum_{i=1}^{len(\mathbf{z})}z_{i}\log{z_{i}}$. Under these conditions, on the basis of induction from Eq. (\ref{eq:ce:middle}) to Eq. (\ref{eq:ce:final}), we can reformulate $\tilde{B}_{\Psi(\mathbf{z})}(q(\mathbf{y}|x),p_{\theta}(\mathbf{y}|x))$ as follows:
\begin{equation}
\begin{aligned}
 &\tilde{B}_{\Psi(\mathbf{z})}(q(\mathbf{y}|x),p_{\theta}(\mathbf{y}|x))\\
 =&-\sum_{x,y} \biggl[\sum_{i=1}^{|Y|} (1-\lambda) p_{d}(y_{i}|x) \log{p_{\theta}(y_{i}|x)}\biggr.\\
 &\hspace{1.5cm}\biggl.+ \sum_{i=1}^{|Y|} \lambda u\{1,|Y|\} \log{p_{\theta}(y_{i}|x)}\biggr]p_{d}(x,y)\\
 =&-\frac{1}{|D|}\sum_{(x,y) \in D} \biggl[ (1-\lambda)\log{p_{\theta}(y|x)}\biggr. \\
 &\hspace{1.5cm}\biggl.+ \sum_{i=1}^{|Y|} \frac{\lambda}{|Y|} \log{p_{\theta}(y_{i}|x)}\biggr].
\end{aligned}
\label{eq:ce:ls}
\end{equation}
The equation in the brackets of Eq. (\ref{eq:ce:ls}) is the cross entropy loss that has a corresponding objective distribution to that of SANS. This loss function is similar in form to SCE with label smoothing (SCE w/ LS) \cite{label1}. This relationship also accords with the theoretical finding that NS can adopt a smoothing to the objective distribution.

\section{Understanding Loss Functions for Fair Comparisons}
\label{sec:interpret}
We summarize the theoretical findings from Sections \ref{sec:bd}, \ref{sec:ns}, and \ref{sec:threl} in Table \ref{tab:overview:loss}. To compare the results from the theoretical findings, we need to understand the differences in their objective distributions and divergences.

\subsection{Objective Distributions}
The objective distributions for NS w/ Uni and SCE are equivalent. We can also see that the objective distribution for SANS is quite similar to that for SCE w/ LS. These theoretical findings will be important for making a fair comparison between scoring methods trained with the NS and SCE loss functions. When a dataset contains low-frequency entities, SANS and SCE w/ LS can improve the link-prediction performance through their smoothing effect, even if there is no performance improvement from the scoring method itself. For comparing the SCE and NS loss functions fairly, therefore, it is necessary to use the vanilla SCE against NS w/ Uni and use SCE w/ LS against SANS.

However, we still have room to discuss the relationship between SANS and SCE w/ LS because $\lambda$ in SANS increases from zero during training, whereas $\lambda$ in SCE w/ LS is fixed. To introduce the behavior of $\lambda$ in SANS to SCE w/ LS, we tried a simple approach in our experiments that trains KGE models via SCE w/ LS using pre-trained embeddings from SCE as initial parameters. Though this approach is not exactly equivalent to SANS, we expected it to work similarly to increasing $\lambda$ from zero in training.

We also discuss the relationship between NS w/ Freq and SCE w/ BC. While NS w/ Freq is often used for learning word embeddings, neither NS w/ Freq nor SCE w/ BC has been explored in KGE. We investigated whether these loss functions are effective in pre-training KGE models\footnote{As a preliminary experiment, we also trained KGE models via NS w/ Freq and SCE w/ BC. However, these methods did not improve the link-prediction performance because frequency-based noise changes the data distribution drastically.}. Because SANS and SCE w/ LS are similar methods to NS w/ Freq and SCE w/ BC in terms of smoothing, in our experiments, we also compared NS w/ Freq with SANS and SCE w/ BC with SCE w/ LS as pre-training methods.

\begin{table*}[h]
\centering
\resizebox{0.885\textwidth}{!}{
\small
\begin{tabular}{llcccccccc}
\toprule
\multirow{2}{*}{\textbf{Method}} & \multirow{2}{*}{\textbf{Loss}} & \multicolumn{4}{c}{\textbf{FB15k-237}} & \multicolumn{4}{c}{\textbf{WN18RR}} \\
\cmidrule(l{2pt}r{2pt}){3-6}\cmidrule(l{2pt}r{2pt}){7-10}
 & & \textbf{MRR} & \textbf{Hits@1} & \textbf{Hits@3} & \textbf{Hits@10} & \textbf{MRR} & \textbf{Hits@1} & \textbf{Hits@3} & \textbf{Hits@10} \\
\midrule
\multirow{4}{*}{TuckER} & NS & 0.257 & 0.151 & 0.297 & 0.472 & 0.431 & 0.407 & 0.440 & 0.473 \\
 & SANS & 0.330 & 0.238 & 0.365 & 0.512 & 0.445 & 0.421 & 0.455 & 0.489 \\
\cmidrule{2-10}
 & SCE & 0.338 & 0.246 & 0.372 & 0.521 & 0.453 & 0.424 & 0.465 & 0.507 \\
 & SCE w/ LS & 0.343 & 0.251 & 0.378 & 0.529 & 0.472 & 0.441 & 0.483 & 0.528 \\
\midrule
\multirow{4}{*}{RESCAL} & NS & 0.337 & 0.247 & 0.368 & 0.516 & 0.385 & 0.354 & 0.405 & 0.437 \\
 & SANS & 0.339 & 0.249 & 0.372 & 0.520 & 0.389 & 0.363 & 0.404 & 0.434 \\
\cmidrule{2-10}
 & SCE & 0.352 & 0.260 & 0.387 & 0.537 & 0.451 & 0.417 & 0.470 & 0.512 \\
 & SCE w/ LS & 0.363 & 0.269 & 0.400 & 0.548 & 0.469 & 0.435 & 0.485 & 0.529 \\
\midrule
\multirow{4}{*}{ComplEx} & NS & 0.296 & 0.211 & 0.324 & 0.468 & 0.394 & 0.373 & 0.403 & 0.432 \\
 & SANS & 0.300 & 0.214 & 0.328 & 0.472 & 0.432 & 0.407 & 0.442 & 0.480 \\
\cmidrule{2-10}
 & SCE & 0.300 & 0.218 & 0.326 & 0.466 & 0.463 & 0.434 & 0.473 & 0.521 \\
 & SCE w/ LS & 0.318 & 0.231 & 0.348 & 0.493 & 0.477 & 0.441 & 0.491 & 0.546 \\
\midrule
\multirow{4}{*}{DistMult} & NS & 0.304 & 0.219 & 0.336 & 0.470 & 0.389 & 0.374 & 0.394 & 0.416 \\
 & SANS & 0.320 & 0.234 & 0.352 & 0.489 & 0.410 & 0.386 & 0.419 & 0.452 \\
\cmidrule{2-10}
 & SCE & 0.342 & 0.252 & 0.374 & 0.521 & 0.438 & 0.407 & 0.447 & 0.497 \\
 & SCE w/ LS & 0.344 & 0.254 & 0.377 & 0.526 & 0.448 & 0.410 & 0.460 & 0.527 \\
\midrule
\multirow{4}{*}{TransE} & NS & 0.284 & 0.182 & 0.319 & 0.498 & 0.218 & 0.011 & 0.390 & 0.510 \\
 & SANS & 0.328 & 0.230 & 0.365 & 0.525 & 0.219 & 0.016 & 0.394 & 0.514 \\
\cmidrule{2-10}
 & SCE & 0.324 & 0.232 & 0.359 & 0.508 & 0.229 & 0.054 & 0.366 & 0.523 \\
 & SCE w/ LS & 0.323 & 0.231 & 0.359 & 0.508 & 0.229 & 0.054 & 0.369 & 0.522 \\
\midrule
\multirow{4}{*}{RotatE} & NS & 0.301 & 0.203 & 0.333 & 0.505 & 0.469 & 0.429 & 0.484 & 0.547 \\
 & SANS & 0.333 & 0.238 & 0.371 & 0.523 & 0.472 & 0.431 & 0.487 & 0.550 \\
\cmidrule{2-10}
 & SCE & 0.315 & 0.228 & 0.347 & 0.486 & 0.452 & 0.423 & 0.463 & 0.507 \\
 & SCE w/ LS & 0.315 & 0.228 & 0.346 & 0.489 & 0.447 & 0.417 & 0.461 & 0.502 \\
\bottomrule
\end{tabular}}
\caption{Results for each method in FB15k-237 and WN18RR datasets. Notations are same as those in Table \ref{tab:overview:loss}.}
 \label{tab:results:loss}
\end{table*}

\subsection{Divergences}
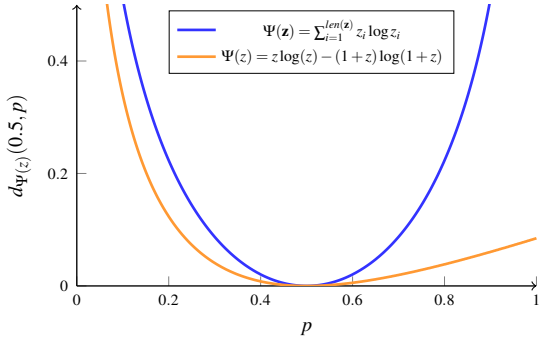
\begin{figure}[t]
 \centering
 \resizebox{0.95\columnwidth}{!}{
 \begin{tikzpicture}
\begin{axis}[width=8.75cm, height=6cm, ticklabel style = {font=\scriptsize}, xlabel={\footnotesize $p$}, xlabel near ticks, ylabel near ticks, ylabel={\footnotesize$d_{\Psi(z)}(0.5,p)$}, xmin=0.0, xmax=1.0, ymin=0, ymax=0.5, samples=500, legend cell align={center},
legend entries={{\scriptsize $\Psi(\mathbf{z})=\sum_{i=1}^{len(\mathbf{z})}z_{i}\log{z_{i}}$},{\scriptsize $\Psi(z)=z\log(z)-(1+z)\log(1+z)$}},
legend style={at={(0.20,0.98)}, anchor=north west, draw=black},
axis lines*=left,
axis line style={->,line width=0.6pt}]
 \addplot[blue!80!white, very thick, domain=0.0:1.0] {ln(0.5)-0.5*ln(x)-0.5*ln(1-x)};
 \addplot[orange!80!white, very thick, domain=0.0:1.0] {(0.5*ln(0.5)-(1+0.5)*ln(1+0.5))-(x*ln(x)-(1+x)*ln(1+x))-(ln(x)-ln(1+x))*(0.5-x)};
\end{axis}
\end{tikzpicture}}
 \caption{Divergence between $0.5$ and $p$ in $d_{\Psi(z)}$ for each $\Psi (z)$. }
 \label{fig:divergence}
\end{figure}

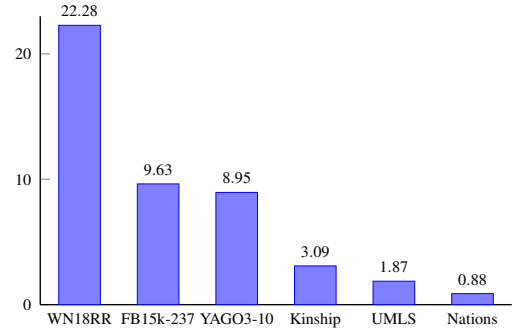
\begin{figure}[t]
 \centering
 \resizebox{0.875\columnwidth}{!}{
 \begin{tikzpicture}
 \begin{axis}[width=8.75cm, height=6cm, ticklabel style = {font=\scriptsize},
 symbolic x coords={WN18RR, FB15k-237, YAGO3-10, Kinship, UMLS, Nations},
 nodes near coords,
 every node near coord/.append style={font=\scriptsize},
 axis lines*=left,
 bar width=18pt,
 xtick=data,ymin=0,ymax=23,
 axis line style={line width=0.6pt}
]
 \addplot[ybar,fill=blue!50!white,draw=blue] coordinates {
 (WN18RR,22.2761416) (FB15k-237,9.629901549) (YAGO3-10,8.949933773) (Kinship,3.089826204) (UMLS,1.868998551) (Nations,0.880179032)
 };
 \end{axis}
\end{tikzpicture}}
 \caption{KL divergence of $p_d(y|x)$ between training and test relations for each dataset}
 \label{fig:kl}
\end{figure}

Comparing $\Psi(z)$ for NS and SCE losses is as important as focusing on their objective distributions. The $\Psi(z)$ determines the distance between model-predicted and data distributions in the loss. It has an important role in determining the behavior of the model. Figure \ref{fig:divergence} shows the distance in Eq. (\ref{eq:bd:reform}) between the probability $p$ and probability $0.5$ for each $\Psi$ in Table \ref{tab:overview:loss}\footnote{In this setting, we can expand $\Psi(\mathbf{z})=\sum_{i=1}^{len(\mathbf{z})}z_{i}\log{z_{i}}$ to $\Psi(z)=z\log{z} + (1-z)\log{(1-z)}$.}. As we can see from the example, $d_{\Psi(z)}(0.5,p)$ of the SCE loss has a larger distance than that of the NS loss. In fact, \citet{8930624} proved that the upper bound of the Bregman divergence for binary labels when $\Psi(\mathbf{z})=\sum_{i=1}^{len(\mathbf{z})}z_{i}\log{z_{i}}$. This means that the SCE loss imposes a larger penalty on the same predicted value than the NS loss when the value of the learning target is the same between the two losses\footnote{See Appendix. \ref{app:proof:bound} for the further details.}.

However, this does not guarantee that the distance of SCE is always larger than NS. This is because the values of the learning target between the two losses are not always the same. To take into account the generally satisfied property, we also focus on the convexity of the functions. In each training instance, the first-order and second-order derivatives of these loss functions indicate that SCE is convex, but NS is not in their domains\footnote{\newcite{goldberg2014word2vec} discuss the convexity of the inner product in NS. Different from theirs, our discussion is about the convexity of the loss functions itself.}. Since this property is independent of the objective distribution, we can consider SCE fits the model more strongly to the training data in general. Because of these features, SCE can be prone to overfitting.

Whether the overfitting is a problem depends on how large the difference between training and test data is. To measure the difference between training and test data in a KG dataset, we calculated the Kullback-Leibler (KL) divergence for $p(\mathbf{y}|x)$ between the training and test data of commonly used KG datasets. To compute $p(y|x)$, we first calculated $p(e_{i}|r_{k},e_{j})=p(e_{i}|r_{k})+p(e_{i}|e_{j})$ on the basis of frequencies in the data then calculated $p(e_{j}|r_{k},e_{i})$ in the same manner. We treated both $p(e_{i}|r_{k},e_{j})$ and $p(e_{j}|r_{k},e_{i})$ as $p(y|x)$. We denote $p(\mathbf{y}|x)$ in the training data as $P$ and in the test data as $Q$. With these notations, we calculated $D_{KL}(P||Q)$ as the KL divergence for $p(\mathbf{y}|x)$ between the test and training data. Figure \ref{fig:kl} shows the results. There is a large difference in the KL divergence between FB15k-237 and WN18RR. We investigated how this difference affects the SCE and NS loss functions for learning KGE models.

In a practical setting, the loss function's divergence is not the only factor to affect the fit to the training data. Model selection also affects the fitting. However, understanding a model's behavior is difficult due to the complicated relationship between model parameters. For this reason, we experimentally investigated which combinations of models and loss functions are suitable for link prediction.

\section{Experiments and Discussion}
We conducted experiments to investigate the validity of what we explained in Section \ref{sec:interpret} through a comparison of the NS and SCE losses.

\subsection{Experimental Settings}
We evaluated the following models on the FB15k-237 and WN18RR datasets in terms of the Mean Reciprocal Rank (MRR), Hits@1, Hits@3, and Hits@10 metrics: TuckER~\cite{balazevic-etal-2019-tucker}; RESCAL~\cite{10.5555/2900423.2900470}; ComplEx~\cite{DBLP:conf/icml/TrouillonWRGB16}; DistMult~\cite{yang2015embedding}; TransE~\cite{NIPS2013_1cecc7a7}; RotatE~\cite{DBLP:journals/corr/abs-1902-10197}. We used LibKGE~\cite{libkge}\footnote{\url{https://github.com/uma-pi1/kge}} as the implementation. For each model to be able to handle queries in both directions, we also trained a model for the reverse direction that shares the entity embeddings with the model for the forward direction.

To determine the hyperparameters of these models, for RESCAL, ComplEx, DistMult, and TransE with SCE and SCE w/ LS, we used the settings that achieved the highest performance in a previous study \cite{Ruffinelli2020You} for each loss function as well as the settings from the original papers for TuckER and RotatE. In TransE with NS and SANS, we used the settings used by \citet{DBLP:journals/corr/abs-1902-10197}. When applying SANS, we set $\alpha$ to an initial value of 1.0 for LibKGE for all models except TransE and RotatE, and for TransE and RotatE, where we followed the settings of the original paper since SANS was used in it. When applying SCE w/ LS, we set $\lambda$ to the initial value of LibKGE, 0.3, except on TransE and RotatE. In the original setting of RotatE, because the values of SANS for TransE and RotatE were tuned, we also selected $\lambda$ from \{0.3, 0.1, 0.01\} using the development data in TransE and RotatE for fair comparison. Appendix \ref{app:params} in the supplemental material details the experimental settings.

\subsection{Characteristics of Loss Functions}
Table \ref{tab:results:loss} shows the results for each loss and model combination. In the following subsections, we discuss investigating whether our findings work in a practical setting on the basis of the results.

\subsubsection{Objective Distributions}
In terms of the objective distribution, when SCE w/ LS improves performance, SANS also improves performance in many cases. Moreover, it accords with our finding that SCE w/ LS and SANS have similar effects. For TransE and RotatE, the relationship does not hold, but as we will see later, this is probably because TransE with SCE and RotatE with SCE did not fit the training data. If the SCE does not fit the training data, the effect of SCE w/ LS is suppressed as it has the same effect as smoothing.

\subsubsection{Divergences}
Next, let us focus on the distance of the loss functions. A comparison of the results of WN18RR and FB15k-237 shows no performance degradation of SCE compared with NS. This indicates that the difference between the training and test data in WN18RR is not so large to cause overfitting problems for SCE.

In terms of the combination of models and loss functions, the results of NS are worse than those of SCE in TuckER, RESCAL, ComplEx, and DistMult. Because the four models have no constraint to prevent fitting to the training data, we consider that the lower scores are caused by underfitting. This conjecture is on the basis that the NS loss weakly fits model-predicted distributions to training-data distributions compared with the SCE loss in terms of divergence and convexity.

In contrast, the performance gap between NS and SCE is smaller in TransE and RotatE. This is because the score functions of TransE and RotatE have bounds and cannot express positive values. Since SCE has a normalization term, it is difficult to represent values close to 1 when the score function cannot represent positive values. This feature prevents TransE and RotatE from completely fitting to the training data. Therefore, we can assume that NS can be a useful loss function when the score function is bounded.

\begin{table}[t]
\resizebox{\columnwidth}{!}{
\begin{tabular}{llcccc}
\toprule
\multicolumn{6}{c}{\textbf{FB15k-237}} \\
\midrule
\textbf{Method} & \textbf{Pre-train} & \textbf{MRR} & \textbf{Hits@1} & \textbf{Hits@3} & \textbf{Hits@10} \\
\midrule
 & - & 0.363 & 0.269 & 0.400 & 0.548 \\
RESCAL & SCE & 0.363 & 0.268 & 0.400 & 0.552 \\
+SCE w/ LS & SCE w/ BC & 0.361 & 0.266 & 0.398 & 0.547 \\
 & SCE w/ LS & 0.364 & 0.269 & 0.402 & 0.550 \\
\midrule
 & - & 0.339 & 0.249 & 0.372 & 0.520 \\
RESCAL & NS & 0.342 & 0.251 & 0.376 & 0.524 \\
+SANS & NS w/ Freq & 0.343 & 0.251 & 0.378 & 0.524 \\
 & SANS & 0.345 & 0.254 & 0.380 & 0.525 \\
\midrule
\multicolumn{6}{c}{\textbf{WN18RR}} \\
\midrule
\textbf{Method} & \textbf{Pre-train} & \textbf{MRR} & \textbf{Hits@1} & \textbf{Hits@3} & \textbf{Hits@10} \\
\midrule
 & - & 0.477 & 0.441 & 0.491 & 0.546 \\
ComplEx & SCE & 0.477 & 0.439 & 0.493 & 0.550 \\
+SCE w/ LS & SCE w/ BC & 0.469 & 0.433 & 0.486 & 0.533 \\
 & SCE w/ LS & 0.481 & 0.444 & 0.496 & 0.553 \\
\midrule
 & - & 0.472 & 0.431 & 0.487 & 0.550 \\
RotatE & NS & 0.470 & 0.433 & 0.483 & 0.548 \\
+SANS & NS w/ Freq & 0.470 & 0.428 & 0.484 & 0.553 \\
 & SANS & 0.471 & 0.429 & 0.488 & 0.552 \\
\bottomrule
\end{tabular}}
\caption{Results of pre-training methods. + denotes combination of model and loss function. Other notations are same as those in Table \ref{tab:overview:loss}.}
 \label{tab:results:pretrain}
\end{table}

\subsection{Effectiveness of Pre-training Methods}
We also explored pre-training for learning KGE models. We selected the methods in Table \ref{tab:results:loss} that achieved the best MRR for each NS-based loss and each SCE-based loss in each dataset. In accordance with the success of word2vec, we chose unigram noise for both NS w/ Freq and SCE w/ BC.

Table \ref{tab:results:pretrain} shows the results. Contrary to our expectations, SCE w/ BC does not work well as a pre-training method. Because the unigram noise for SCE w/ BC can drastically change the original data distribution, SCE w/ BC is thought to be effective when the difference between training and test data is large. However, since the difference is not so large in the KG datasets, as discussed in the previous subsection, we believe that the unigram noise may be considered unsuitable for these datasets.

Compared with SCE w/ BC, both SCE w/ LS and SANS are effective for pre-training. This is because the hyperparameters of SCE w/ LS and SANS are adjusted for KG datasets.

When using vanilla SCE as a pre-training method, there is little improvement in prediction performance, compared with other methods. This result suggests that increasing $\lambda$ in training is not as important for improving task performance.

For RotatE, there is no improvement in pre-training. Because RotatE has strict constraints on its relation representation, we believe it may degrade the effectiveness of pre-training.

\section{Related Work}
\newcite{DBLP:journals/corr/MikolovSCCD13} proposed the NS loss function as an approximation of the SCE loss function to reduce computational cost and handle a large vocabulary for learning word embeddings. NS is now used in various NLP tasks, which must handle a large amount of vocabulary or labels. \newcite{melamud-etal-2017-simple} used the NS loss function for training a language model. \newcite{DBLP:conf/icml/TrouillonWRGB16} introduced the NS loss function to KGE. In contextualized pre-trained embeddings, \newcite{clark-etal-2020-pre} indicated that ELECTRA \cite{Clark2020ELECTRA:}, a variant of BERT \cite{devlin-etal-2019-bert}, follows the same manner of the NS loss function.

NS is frequently used to train KGE models. KGE is a task to complement a knowledge graph that describes relationships between entities. Knowledge graphs are used in various important downstream tasks because of its convenience in incorporating external knowledge, such as in a language model \cite{logan-etal-2019-baracks}, dialogue \cite{moon-etal-2019-opendialkg}, question-answering \cite{10.1145/3038912.3052675}, natural language inference \cite{k-m-etal-2018-learning}, and named entity recognition \cite{He_Wu_Yin_Cai_2020}. Thus, current KGE is important in NLP.

Due to the importance of KGE, various scoring methods including RESCAL~\cite{10.5555/2900423.2900470}, TransE~\cite{NIPS2013_1cecc7a7}, DistMult~\cite{yang2015embedding}, ComplEx~\cite{DBLP:conf/icml/TrouillonWRGB16}, TuckER~\cite{balazevic-etal-2019-tucker}, and RotatE~\cite{DBLP:journals/corr/abs-1902-10197} used in our experiment, have been proposed. However, the relationship between these score functions and loss functions is not clear. Several studies~\cite{Ruffinelli2020You,ali2020pykeen} have investigated the best combinations of scoring method, loss function, and their hyperparameters in KG datasets. These studies differ from ours in that they focused on empirically searching for good combinations rather than theoretical investigations.

As a theoretical study, \newcite{10.5555/2969033.2969070} showed that NS is equivalent to factorizing a matrix for PMI when a unigram distribution is selected as a noise distribution. \newcite{nce-ns} investigated the difference between NCE \cite{nce} and NS. \newcite{10.5555/3020548.3020582} revealed that NCE is derivable from Bregman divergence. Our derivation for NS is inspired by their work. \newcite{meister-etal-2020-generalized} proposed a framework to jointly interpret label smoothing and confidence penalty \cite{DBLP:journals/corr/PereyraTCKH17} through investigating their divergence. \newcite{10.1145/3394486.3403218} theoretically induced that a noise distribution that is close to the true distribution behind the training data is suitable for training KGE models in NS. They also proposed a variant of SANS in the basis of their investigation.

Different from these studies, we investigated the distributions at optimal solutions of SCE and NS loss functions while considering several types of noise distribution in NS.

\section{Conclusion}
We revealed the relationships between SCE and NS loss functions in KGE. Through theoretical analysis, we showed that SCE and NS w/ Uni are equivalent in objective distribution, which is the predicted distribution of a model at an optimal solution, and that SCE w/ LS and SANS have similar objective distributions. We also showed that SCE more strongly fits a model to the training data than NS due to the divergence and convexity of SCE.

The experimental results indicate that the differences in the divergence of the two losses were not large enough to affect dataset differences. The results also indicate that SCE works well with highly flexible scoring methods, which do not have any bound of the scores, while NS works well with RotatE, which cannot express positive values due to its bounded scoring. Moreover, they indicate that SCE and SANS work better in pre-training than NS w/ Uni, commonly used for learning word embeddings.

For future work, we will investigate the properties of loss functions in out-of-domain data.

\section*{Acknowledgements}
This work was partially supported by
JSPS Kakenhi Grant nos.~19K20339, 21H03491, and 21K17801.

\bibliography{acl2021}
\bibliographystyle{acl_natbib}

\appendix
\onecolumn
\section{Proof of Proposition 1, 2, and 3}
\label{breg:ns:details}
We can reformulate $\ell_{NS}$ as follows:
\begin{align}
    \ell^{NS}(\theta)&=-\frac{1}{|D|}
    \sum_{(x,y) \in D}\left(\log(P(C=1,y|x;\theta))+\sum_{i=1,y_{i}\sim p_n}^{\nu}\log(P(C=0,y_{i}|x;\theta))\right)\nonumber\\
    &=-\frac{1}{|D|}\sum_{(x,y) \in D}\log(P(C=1,y|x;\theta))-\frac{1}{|D|}\sum_{(x,y)\in D}\sum_{i=1,y_{i}\sim p_n}^{\nu}\log(P(C=0,y_{i}|x;\theta))\nonumber\\
    &=-\frac{1}{|D|}\sum_{(x,y) \in D}\log(\frac{1}{1+G(y|x;\theta)})-\frac{1}{ |D|}\sum_{(x,y)\in D}\sum_{i=1,y_{i}\sim p_n}^{\nu}\log(\frac{G(y_{i}|x;\theta)}{1+G(y_{i}|x;\theta)})\nonumber\\
    &=\frac{1}{|D|}\sum_{(x,y)\in D}\log(1+G(y|x;\theta)) + \frac{\nu}{\nu |D|}\sum_{(x,y)\in D}\sum_{i=1,y_{i}\sim p_n}^{\nu}\log(1+\frac{1}{G(y_i|x;\theta)})\nonumber\\
    &=\sum_{x,y} p_{d}(y|x)\log(1+G(y|x;\theta))p_{d}(x) + \sum_{x,y} \nu p_{n}(y|x)\log(1+\frac{1}{G(y|x;\theta)})p_{d}(x)\label{app:eq:nce:int}
\end{align}
Letting $u=(x,y)$, $f(u)=\frac{\nu p_{n}(y|x)}{p_{d}(y|x)}$, $g(u)=G(y|x;\theta)$, and $p_{d}(x)=\frac{1}{p_{d}(y|x)}p_{d}(x,y)$, we can reformulate Eq.~(\ref{app:eq:nce:int}) as:
\begin{align}
    \ell^{NS}(\theta)=&\left(\sum_{x,y} p_{d}(y|x)\log(1+g(u))\frac{1}{p_{d}(y|x)}p_{d}(x,y) + \sum_{x,y} \nu p_{n}(y|x)\log(1+\frac{1}{g(u)})\frac{1}{p_{d}(y|x)}p_{d}(x,y)\right)\nonumber\\
    =&\sum_{x,y} \left[\log(1+g(u)) + \log(1+\frac{1}{g(u)})f(u)\right]p_{d}(x,y)\nonumber\\
    =&\sum_{x,y} \Bigl[\log(1+g(u)) - \log(g(u))f(u) + \log(1+g(u))f(u)\Bigr]p_{d}(x,y)\nonumber\\
    =&\sum_{x,y} \Bigl [-g(u)\log(1+g(u)) + (1+g(u))\log(1+g(u))\Bigr.\nonumber\\
    & + \log(g(u))g(u)  + \log(1+g(u))g(u) - \log(g(u))f(u) + \log(1+g(u))f(u)\Bigl.\Bigr]p_{d}(x,y)\label{app:eq:nce:exp}
\end{align}
With $\Psi(g(u))=g(u)\log(g(u))-(1+g(u))\log(1+g(u))$ and $\nabla\Psi(g(u))=\log(g(u))-\log(1+g(u))$, we can reformulate Eq.~(\ref{app:eq:nce:exp}) as:
\begin{align}
    \ell^{NS}(\theta)=& \sum_{x,y}\Bigl[-\Psi(g(u))+\nabla\Psi(g(u))g(u)-\nabla\Psi(g(u))f(u)\Bigr]p_{d}(x,y)\nonumber\\
    =& \tilde{B}_{\Psi(z)}(f(u),g(u)).
    \label{app:eq:ns:breg}
\end{align}
From Eq.~(\ref{app:eq:ns:breg}), when $\ell^{NS}(\theta)$ is minimized, $g(u)=f(u)$ is satisfied. In this condition, $G(y|x;\theta)$ becomes $\frac{\nu p_{n}(y|x)}{p_{d}(y|x)}$, and $\exp(f_{\theta}(x,y))$ becomes $\frac{p_{d}(y|x)}{\nu p_{n}(y|x)}$ as follows:
\begin{equation}
    g(u)=f(u) \Leftrightarrow G(y|x;\theta)=\frac{\nu p_{n}(y|x)}{p_{d}(y|x)} \Leftrightarrow \exp(f_{\theta}(x,y)) = \frac{p_{d}(y|x)}{\nu p_{n}(y|x)}\label{app:eq:ns:obj-true}.
\end{equation}
Based on the Eq.~(\ref{eq:softmax}) and  Eq.~(\ref{app:eq:ns:obj-true}), the objective distribution for $p_{\theta}(y|x)$ is as follows:
\begin{align}
    p_{\theta}(y|x) = \frac{p_{d}(y|x)}{p_{n}(y|x)\sum\limits_{y_{i}\in Y}\frac{p_{d}(y_i|x)}{p_{n}(y_i|x)}}\label{app:eq:ns:softmax}.
\end{align}

\section{Proof of Proposition 4}
\label{app:proof:pmi}
PMI is induced by multiplying $p_{d}(x)$ to the right-hand side of Eq.~(\ref{eq:ns:obj-true}) and then computing logarithm for both sides as follows:
\begin{equation}
G(y|x;\theta)=\frac{p_{n}(y|x)}{p_{d}(y|x)} \Leftrightarrow
     \exp(f_{\theta}(x,y)) = \frac{p_{d}(y|x)}{p_{n}(y|x)} = \frac{p_{d}(y|x)}{p_{d}(y)} = \frac{p_{d}(x,y)}{p_{d}(x)p_{d}(y)} \Leftrightarrow f_{\theta}(x,y) = \log\frac{p_{d}(x,y)}{p_{d}(y)p_{d}(y)}
\end{equation}

\section{Proof of Proposition 5}
\label{app:proof:uninoise}
When $p_{n}(y|x)$ is a uniform distribution, $p_{n}(y|x)\sum\limits_{y_{i}\in Y}\frac{p_{d}(y_i|x)}{p_{n}(y_i|x)}=\sum\limits_{y_{i}\in Y}p_{d}(y_i|x)=1$, and thus, Eq.~(\ref{eq:ns:softmax}) becomes $p_{d}(y|x)$.

\section{Experimental Details}
\label{app:params}

\noindent\textbf{Dataset}: We use FB15k-237~\cite{toutanova-chen-2015-observed}\footnote{\url{https://www.microsoft.com/en-us/download/confirmation.aspx?id=52312}} and WN18RR~\cite{dettmers2018conve}\footnote{\url{https://github.com/TimDettmers/ConvE}} datasets in the experiments. We followed the standard split in the original papers for each dataset. Table \ref{tab:app:data} lists the statistics for each dataset.
\begin{table}[h!]
    \centering
    \small
    \begin{tabular}{llllll}
    \toprule
    \multirow{2}{*}{Dataset}&\multirow{2}{*}{Entities}&\multirow{2}{*}{Relations}&\multicolumn{3}{c}{Tuples}\\
    \cmidrule(lr){4-6}
 &  &  & Train & Valid & Test\\
    \midrule
WN18RR &40,943 &11 &86,835 &3,034 &3,134\\
FB15k-237 &14,541 &237 &272,115 &17,535& 20,466  \\
\bottomrule
    \end{tabular}
    \caption{The numbers of each instance for each dataset.}
    \label{tab:app:data}
\end{table}

\noindent\textbf{Metric}: We evaluated the link prediction performance of models with MRR, Hits@1, Hits@3, and Hits@10 by ranking test triples against all other triples not appeared in the training, valid, and test datasets. We used LibKGE for calculating these metric scores.

\noindent\textbf{Model}: We compared the following models:
TuckER~\cite{balazevic-etal-2019-tucker}; RESCAL~\cite{10.5555/2900423.2900470}; ComplEx~\cite{DBLP:conf/icml/TrouillonWRGB16}; DistMult~\cite{yang2015embedding}; TransE~\cite{NIPS2013_1cecc7a7}; RotatE~\cite{DBLP:journals/corr/abs-1902-10197}.
For each model, we also trained a model for the reverse direction that shares the entity embeddings with the model for the forward direction. Thus, the dimension size of subject and object embeddings are the same in all models.

\noindent\textbf{Implementation}: We used LibKGE~\cite{libkge}\footnote{\url{https://github.com/uma-pi1/kge}} as the implementation.
We used its 1vsAll setting for SCE-based loss functions and negative sampling setting for NS-based loss functions.
We modified LibKGE to be able to use label smoothing on the 1vsAll setting.
We also incorporated NS w/ Freq and SCE w/ BC into the implementation.

\noindent\textbf{Hyper-parameter}: Table \ref{tab:hp:fb15k-237} and \ref{tab:hp:wn18rr} show the hyper-parameter settings of each method for each dataset.
In RESCAL, ComplEx, and DistMult we used the settings that achieved the highest performance for each loss function in the previous study \cite{Ruffinelli2020You}\footnote{\url{https://github.com/uma-pi1/kge-iclr20}}.
In TuckER and RotatE, we follow the settings from the original paper. When applying SANS, we set $\alpha$ to an initial value of 1.0 for LibKGE for all models except TransE and RotatE, and for TransE and RotatE, where we followed the settings of the original paper of SANS since SANS was used in it.
When applying SCE w/ LS, we set $\lambda$ to the initial value of LibKGE, 0.3, except on TransE and RotatE.
In the original setting of TransE and RotatE, because the value of SANS was tuned for comparison, for fairness, we selected $\lambda$ from \{0.3, 0.1, 0.01\} by using the development data through a single run for each value.
We set the maximum epoch to 800.
We calculated MRR every five epochs on the developed data, and the training was terminated when the highest value was not updated ten times.
We chose the best model by using the MRR score on the development data. 
These hyperparameters were also used in the pre-training step.

\begin{table}[h!]
\resizebox{\columnwidth}{!}{
\begin{tabular}{llllllllllllllllll}
\toprule
 &
   &
  \multicolumn{16}{c}{\textbf{FB15k-237}} \\
  \midrule
\multicolumn{2}{c}{\multirow{2}{*}{Model}} &
  \multirow{2}{*}{Batch} &
  \multirow{2}{*}{Dim} &
  \multirow{2}{*}{Initialize} &
  \multicolumn{3}{l}{Regularize} &
  \multicolumn{2}{l}{Dropout} &
  \multicolumn{4}{l}{Optimizer} &
  \multicolumn{2}{l}{Sample} &
  \multirow{2}{*}{$\lambda$} &
  \multirow{2}{*}{$\alpha$} \\
  \cmidrule(lr){6-8}\cmidrule(lr){9-10}\cmidrule(lr){11-14}\cmidrule(lr){15-16}
\multicolumn{2}{l}{} &
   &
   &
   &
  Type &
  Entity &
  Relation &
  Entity &
  Rel. &
  Type &
  LR &
  Decay &
  P. &
  sub. &
  obj. &
   &
   \\
   \midrule
\multirow{4}{*}{TuckER} &
  SCE &
  128 &
  200 &
  xn: 1.0 &
  - &
  - &
  - &
  0.3 &
  0.4 &
  Adam &
  0.0005 &
  - &
  - &
  All &
  All &
  - &
  - \\
 &
  SCE w/ LS &
  128 &
  200 &
  xn: 1.0 &
  - &
  - &
  - &
  0.3 &
  0.4 &
  Adam &
  0.0005 &
  - &
  - &
  All &
  All &
  0.3 &
  - \\
 &
  NS &
  128 &
  200 &
  xn: 1.0 &
  - &
  - &
  - &
  0.3 &
  0.4 &
  Adam &
  0.0005 &
  - &
  - &
  All &
  All &
  - &
  - \\
 &
  SANS &
  128 &
  200 &
  xn: 1.0 &
  - &
  - &
  - &
  0.3 &
  0.4 &
  Adam &
  0.0005 &
  - &
  - &
  All &
  All &
  - &
  1.0 \\
  \midrule
\multirow{4}{*}{Rescal} &
  SCE &
  512 &
  128 &
  n: 0.123 &
  - &
  - &
  - &
  0.427 &
  0.159 &
  Adam &
  7.39E-5 &
  0.95 &
  1 &
  All &
  All &
  - &
  - \\
 &
  SCE w/ LS &
  512 &
  128 &
  n: 0.123 &
  - &
  - &
  - &
  0.427 &
  0.159 &
  Adam &
  7.39E-5 &
  0.95 &
  1 &
  All &
  All &
  0.3 &
  - \\
 &
  NS &
  256 &
  128 &
  xn: 1.0 &
  lp: 3 &
  1.22E-12 &
  4.80E-14 &
  0.347 &
  - &
  Adagrad &
  0.0170 &
  0.95 &
  5 &
  22 &
  155 &
  - &
  - \\
 &
  SANS &
  256 &
  128 &
  xn: 1.0 &
  lp: 3 &
  1.22E-12 &
  4.80E-14 &
  0.347 &
  - &
  Adagrad &
  0.0170 &
  0.95 &
  5 &
  22 &
  155 &
  - &
  1.0 \\
  \midrule
\multirow{4}{*}{ComlEx} &
  SCE &
  512 &
  128 &
  u: 0.311 &
  - &
  - &
  - &
  0.0476 &
  0.443 &
  Adagrad &
  0.503 &
  0.95 &
  7 &
  All &
  All &
  - &
  - \\
 &
  SCE w/ LS &
  512 &
  128 &
  u: 0.311 &
  - &
  - &
  - &
  0.0476 &
  0.443 &
  Adagrad &
  0.503 &
  0.95 &
  7 &
  All &
  All &
  0.3 &
  - \\
 &
  NS &
  512 &
  256 &
  n: 4.81E-5 &
  lp: 2 &
  6.34E-9 &
  9.08E-18 &
  0.182 &
  0.0437 &
  Adagrad &
  0.241 &
  0.95 &
  4 &
  1 &
  48 &
  - &
  - \\
 &
  SANS &
  512 &
  256 &
  n: 4.81E-5 &
  lp: 2 &
  6.34E-9 &
  9.08E-18 &
  0.182 &
  0.0437 &
  Adagrad &
  0.241 &
  0.95 &
  4 &
  1 &
  48 &
  - &
  1.0 \\
  \midrule
\multirow{4}{*}{DistMult} &
  SCE &
  512 &
  128 &
  n: 0.806 &
  - &
  - &
  - &
  0.370 &
  0.280 &
  Adam &
  0.00063 &
  0.95 &
  1 &
  All &
  All &
  - &
  - \\
 &
  SCE &
  512 &
  128 &
  n: 0.806 &
  - &
  - &
  - &
  0.370 &
  0.280 &
  Adam &
  0.00063 &
  0.95 &
  1 &
  All &
  All &
  0.3 &
  - \\
 &
  NS &
  1024 &
  256 &
  u: 0.848 &
  lp: 3 &
  1.55E-10 &
  3.93E-15 &
  0.455 &
  0.360 &
  Adagrad &
  0.141 &
  0.95 &
  9 &
  557 &
  367 &
  - &
  - \\
 &
  SANS &
  1024 &
  256 &
  u: 0.848 &
  lp: 3 &
  1.55E-10 &
  3.93E-15 &
  0.455 &
  0.360 &
  Adagrad &
  0.141 &
  0.95 &
  9 &
  557 &
  367 &
  - &
  1.0 \\
  \midrule
\multirow{4}{*}{TransE} &
  SCE &
  128 &
  128 &
  u: 1.0E-5 &
  - &
  - &
  - &
  - &
  - &
  Adam &
  0.0003 &
  0.95 &
  5 &
  All &
  All &
  - &
  - \\
 &
  SCE w/ LS &
  128 &
  128 &
  u: 1.0E-5 &
  - &
  - &
  - &
  - &
  - &
  Adam &
  0.0003 &
  0.95 &
  5 &
  All &
  All &
  0.01 &
  - \\
 &
  NS &
  1024 &
  1000 &
  xu: 1.0 &
  - &
  - &
  - &
  - &
  - &
  Adam &
  0.00005 &
  0.95 &
  5 &
  256 &
  256 &
  - &
  - \\
 &
  SANS &
  1024 &
  1000 &
  xu: 1.0 &
  - &
  - &
  - &
  - &
  - &
  Adam &
  0.00005 &
  0.95 &
  5 &
  256 &
  256 &
  - &
  1.0\\
  \midrule
\multirow{4}{*}{Rotate} &
  SCE &
  1024 &
  1000 &
  xu: 1.0 &
  - &
  - &
  - &
  - &
  - &
  Adam &
  0.00005 &
  0.95 &
  5 &
  All &
  All &
  - &
  - \\
 &
  SCE w/ LS &
  1024 &
  1000 &
  xu: 1.0 &
  - &
  - &
  - &
  - &
  - &
  Adam &
  0.00005 &
  0.95 &
  5 &
  All &
  All &
  0.01 &
  - \\
 &
  NS &
  1024 &
  1000 &
  xu: 1.0 &
  - &
  - &
  - &
  - &
  - &
  Adam &
  0.00005 &
  0.95 &
  5 &
  256 &
  256 &
  - &
  - \\
 &
  SANS &
  1024 &
  1000 &
  xu: 1.0 &
  - &
  - &
  - &
  - &
  - &
  Adam &
  0.00005 &
  0.95 &
  5 &
  256 &
  256 &
  - &
  1.0\\
  \bottomrule
\end{tabular}}
\caption{The hyper-parameters for each model in FB15k-237. Rel. denotes relation, P. denotes patience, sub. denotes subjective, obj. denotes objective, xn denotes xavier normal, n denotes normal, xu denotes xavier uniform, and u denotes uniform.}
\label{tab:hp:fb15k-237}
\end{table}

\begin{table}[h!]
\resizebox{\columnwidth}{!}{
\begin{tabular}{llllllllllllllllll}
\toprule
 &
   &
  \multicolumn{16}{c}{\textbf{WN18RR}} \\
  \midrule
\multicolumn{2}{c}{\multirow{2}{*}{Model}} &
  \multirow{2}{*}{Batch} &
  \multirow{2}{*}{Dim} &
  \multirow{2}{*}{Initialize} &
  \multicolumn{3}{l}{Regularize} &
  \multicolumn{2}{l}{Dropout} &
  \multicolumn{4}{l}{Optimizer} &
  \multicolumn{2}{l}{Sample} &
  \multirow{2}{*}{$\lambda$} &
  \multirow{2}{*}{$\alpha$} \\
  \cmidrule(lr){6-8}\cmidrule(lr){9-10}\cmidrule(lr){11-14}\cmidrule(lr){15-16}
\multicolumn{2}{l}{} &
   &
   &
   &
  Type &
  Entity &
  Relation &
  Entity &
  Rel. &
  Type &
  LR &
  Decay &
  P. &
  sub. &
  obj. &
   &
   \\
   \midrule
\multirow{4}{*}{TuckER} &
  SCE &
  128 &
  200 &
  xn: 1.0 &
  - &
  - &
  - &
  0.2 &
  0.2 &
  Adam &
  0.0005 &
  - &
  - &
  All &
  All &
  - &
  - \\
 &
  SCE w/ LS &
  128 &
  200 &
  xn: 1.0 &
  - &
  - &
  - &
  0.2 &
  0.2 &
  Adam &
  0.0005 &
  - &
  - &
  All &
  All &
  0.3 &
  - \\
 &
  NS &
  128 &
  200 &
  xn: 1.0 &
  - &
  - &
  - &
  0.2 &
  0.2 &
  Adam &
  0.0005 &
  - &
  - &
  All &
  All &
  - &
  - \\
 &
  SANS &
  128 &
  200 &
  xn: 1.0 &
  - &
  - &
  - &
  0.2 &
  0.2 &
  Adam &
  0.0005 &
  - &
  - &
  All &
  All &
  - &
  1.0 \\
  \midrule
\multirow{4}{*}{Rescal} &
  SCE &
  512 &
  256 &
  xn: 1.0 &
  - &
  - &
  - &
  - &
  - &
  Adam &
  0.00246 &
  0.95 &
  9 &
  All &
  All &
  - &
  - \\
 &
  SCE w/ LS &
  512 &
  256 &
  xn: 1.0 &
  - &
  - &
  - &
  - &
  - &
  Adam &
  0.00246 &
  0.95 &
  9 &
  All &
  All &
  0.3 &
  - \\
 &
  NS &
  512 &
  128 &
  n: 1.64E-4 &
  - &
  - &
  - &
  - &
  - &
  Adam &
  0.00152 &
  0.95 &
  1 &
  6 &
  8 &
  - &
  - \\
 &
  SANS &
  512 &
  128 &
  n: 1.64E-4 &
  - &
  - &
  - &
  - &
  - &
  Adam &
  0.00152 &
  0.95 &
  1 &
  6 &
  8 &
  - &
  1.0 \\
  \midrule
\multirow{4}{*}{ComlEx} &
  SCE &
  512 &
  128 &
  u: 0.281 &
  lp: 2 &
  4.52E-6 &
  4.19E-10 &
  0.359 &
  0.311 &
  Adagrad &
  0.526 &
  0.95 &
  5 &
  All &
  All &
  - &
  - \\
 &
  SCE w/ LS &
  512 &
  128 &
  u: 0.281 &
  lp: 2 &
  4.52E-6 &
  4.19E-10 &
  0.359 &
  0.311 &
  Adagrad &
  0.526 &
  0.95 &
  5 &
  All &
  All &
  0.3 &
  - \\
 &
  NS &
  1024 &
  128 &
  xn: 1.0 &
  - &
  - &
  - &
  0.0466 &
  0.0826 &
  Adam &
  3.32E-5 &
  0.95 &
  7 &
  6 &
  6 &
  - &
  - \\
 &
  SANS &
  1024 &
  128 &
  xn: 1.0 &
  - &
  - &
  - &
  0.0466 &
  0.0826 &
  Adam &
  3.32E-5 &
  0.95 &
  7 &
  6 &
  6 &
  - &
  1.0 \\
  \midrule
\multirow{4}{*}{DistMult} &
  SCE &
  512 &
  128 &
  u: 0.311 &
  lp: 2 &
  1.44E-18 &
  1.44E-18 &
  0.0476 &
  0.443 &
  Adagrad &
  0.503 &
  0.95 &
  7 &
  All &
  All &
  - &
  - \\
 &
  SCE w/ LS &
  512 &
  128 &
  u: 0.311 &
  lp: 2 &
  1.44E-18 &
  1.44E-18 &
  0.0476 &
  0.443 &
  Adagrad &
  0.503 &
  0.95 &
  7 &
  All &
  All &
  0.3 &
  - \\
 &
  NS &
  1024 &
  128 &
  xn: 1.0 &
  - &
  - &
  - &
  0.0466 &
  0.0826 &
  Adam &
  3.32E-5 &
  0.95 &
  7 &
  6 &
  6 &
  - &
  - \\
 &
  SANS &
  1024 &
  128 &
  xn: 1.0 &
  - &
  - &
  - &
  0.0466 &
  0.0826 &
  Adam &
  3.32E-5 &
  0.95 &
  7 &
  6 &
  6 &
  - &
  1.0 \\
  \midrule
\multirow{4}{*}{TransE} &
  SCE &
  128 &
  512 &
  xn: 1.0 &
  lp: 2 &
  2.13E-7 &
  8.99E-13 &
  0.252 &
  - &
  Adagrad &
  0.253 &
  0.95 &
  5 &
  All &
  All &
  - &
  - \\
 &
  SCE w/ LS &
  128 &
  512 &
  xn: 1.0 &
  lp: 2 &
  2.13E-7 &
  8.99E-13 &
  0.252 &
  - &
  Adagrad &
  0.253 &
  0.95 &
  5 &
  All &
  All &
  0.01 &
  - \\
 &
  NS &
  512 &
  500 &
  xu: 1.0 &
  - &
  - &
  - &
  - &
  - &
  Adam &
  0.00005 &
  0.95 &
  5 &
  1024 &
  1024 &
  - &
  - \\
 &
  SANS &
  512 &
  500 &
  xu: 1.0 &
  - &
  - &
  - &
  - &
  - &
  Adam &
  0.00005 &
  0.95 &
  5 &
  1024 &
  1024 &
  - &
  0.5\\
  \midrule
\multirow{4}{*}{Rotate} &
  SCE &
  512 &
  500 &
  xu: 1.0 &
  - &
  - &
  - &
  - &
  - &
  Adam &
  0.00005 &
  0.95 &
  5 &
  All &
  All &
  - &
  - \\
 &
  SCE w/ LS &
  512 &
  500 &
  xu: 1.0 &
  - &
  - &
  - &
  - &
  - &
  Adam &
  0.00005 &
  0.95 &
  5 &
  All &
  All &
  0.01 &
  - \\
 &
  NS &
  512 &
  500 &
  xu: 1.0 &
  - &
  - &
  - &
  - &
  - &
  Adam &
  0.00005 &
  0.95 &
  5 &
  1024 &
  1024 &
  - &
  - \\
 &
  SANS &
  512 &
  500 &
  xu: 1.0 &
  - &
  - &
  - &
  - &
  - &
  Adam &
  0.00005 &
  0.95 &
  5 &
  1024 &
  1024 &
  - &
  0.5\\
  \bottomrule
\end{tabular}}
\caption{The hyper-parameters for each model in WN18RR. The notations are the same as Table \ref{tab:hp:fb15k-237}.}
\label{tab:hp:wn18rr}
\end{table}

\noindent\textbf{Validation Score} Table \ref{tab:valid}, \ref{tab:valid:pretraining}, and \ref{tab:valid:pretrained} show the best MRR scores of each loss for each model on the validation dataset.

\begin{table*}[h!]
\centering
\small
\begin{tabular}{llcc}
\toprule
             \textbf{Model}       &  \textbf{Loss} & \textbf{FB15k-237} & \textbf{WN18RR} \\
\midrule
\multirow{4}{*}{TuckER}  & SCE       & 0.345     & 0.451  \\
                         & SCE w/ LS & 0.350     & 0.470  \\
                         & NS        & 0.261     & 0.433  \\
                         & SANS      & 0.337     & 0.441  \\
\midrule
\multirow{4}{*}{RESCAL}  & SCE       & 0.359     & 0.461  \\
                         & SCE w/ LS & 0.369     & 0.474  \\
                         & NS        & 0.344     & 0.389  \\
                         & SANS      & 0.344     & 0.390  \\
\midrule
\multirow{4}{*}{ComplEx} & SCE       & 0.304     & 0.468  \\
                         & SCE w/ LS & 0.324     & 0.478  \\
                         & NS        & 0.302     & 0.399  \\
                         & SANS      & 0.308     & 0.433  \\
\midrule
\multirow{4}{*}{DistMult} & SCE      & 0.350    & 0.441  \\
                         & SCE w/ LS & 0.351    & 0.451  \\
                         & NS        & 0.308    & 0.391  \\
                         & SANS      & 0.326    & 0.412  \\
\midrule
\multirow{4}{*}{TransE} & SCE       & 0.328 & 0.227 \\
                         & SCE w/ LS & 0.322 & 0.220 \\
                         & NS        & 0.289 & 0.216 \\
                         & SANS      & 0.333 & 0.218 \\
\midrule
\multirow{4}{*}{RotatE}  & SCE       & 0.320     & 0.452  \\
                         & SCE w/ LS & 0.320     & 0.449  \\
                         & NS        & 0.306     & 0.472  \\
                         & SANS      & 0.340     & 0.475  \\
\bottomrule
\end{tabular}
\caption{The best MRR scores on validation data.}
\label{tab:valid}
\end{table*}

\begin{table}[h]
\centering
\small
\begin{tabular}{lll}
\toprule
Dataset                    & Mehotd            & MRR   \\
\midrule
\multirow{2}{*}{FB15k-237} & RESCAL+SCE w/BC   & 0.149 \\
                           & RESCAL+NS w/ Freq & 0.171 \\
\midrule
\multirow{2}{*}{WN18RR}    & ComplEx+SCE w/ BC & 0.361 \\
                           & RotatE+NS w/ Freq & 0.469 \\
\bottomrule
\end{tabular}
\caption{The best MRR scores of pre-trained models on validation data.}
\label{tab:valid:pretraining}
\end{table}

\begin{table}[h]
\centering
\small
\begin{tabular}{lll}
\toprule
\multicolumn{3}{c}{\textbf{FB15k-237}}                           \\
\midrule
Method                             & Pretrain   & MRR   \\
\midrule
\multirow{3}{*}{RESCAL+SCE w / LS} & SCE        & 0.369 \\
                                   & SCE w/ BC  & 0.369 \\
                                   & SCE w/ LS  & 0.371 \\
\midrule
\multirow{3}{*}{RESCAL+SANS}       & NS         & 0.349 \\
                                   & NS w/ Freq & 0.348 \\
                                   & SANS       & 0.350 \\
\midrule
\multicolumn{3}{c}{\textbf{WN18RR}}                              \\
\midrule
Method                             & Pretrain   & MRR   \\
\midrule
\multirow{3}{*}{ComplEx+SCE w/ LS} & SCE        & 0.483 \\
                                   & SCE w/ BC  & 0.469 \\
                                   & SCE w/ LS  & 0.481 \\
\midrule
\multirow{3}{*}{RotatE+SANS}       & NS         & 0.472 \\
                                   & NS w/ Freq & 0.474 \\
                                   & SANS       & 0.475 \\
\bottomrule
\end{tabular}
\caption{The best MRR scores of models initialized with pre-trained embeddings on validation data.}
\label{tab:valid:pretrained}
\end{table}

\noindent\textbf{Device}: In all models, we used a single NVIDIA RTX2080Ti for training. Except for RotetE with SCE-based loss functions, all models finished the training in one day.
The RotetE with SCE-based loss function finished the training in at most one week.

\clearpage
\section{Note: the divergence between the NS and SCE loss functions}
\label{app:proof:bound}
\citet{8930624} proved that the upper bound of the Bregman divergence for binary labels when $\Psi(\mathbf{z})=\sum_{i=1}^{len(\mathbf{z})}z_{i}\log{z_{i}}$. However, to compare the SCE and NS loss functions in general, we need to consider the divergence of multi labels in SCE.
When $\Psi(\mathbf{z})=\sum_{i=1}^{len(\mathbf{z})}z_{i}\log{z_{i}}$, we can derive the following inequality by using the log sum inequality:
\begin{align}
     &d_{\Psi(\mathbf{z})}(p_{d}(\mathbf{y}|x),p_{\theta}(\mathbf{y}|x)) \nonumber\\
     =& \sum_{i=1}^{|Y|} p_{d}(y_{i}|x)\log\frac{p_{d}(y_{i}|x)}{p_{\theta}(y_{i}|x)} \nonumber\\
     =& p_{d}(y_{j}|x)\log\frac{p_{d}(y_{j}|x)}{p_{\theta}(y_{j}|x)} + \sum_{i \not = j}^{|Y|} p_{d}(y_{i}|x)\log\frac{p_{d}(y_{i}|x)}{p_{\theta}(y_{i}|x)}\nonumber\\
     \geqq & p_{d}(y_{j}|x)\log\frac{p_{d}(y_{j}|x)}{p_{\theta}(y_{j}|x)}  + (\sum_{i \not = j}^{|Y|} p_{d}(y_{i}|x))\log\frac{(\sum_{i \not = j}^{|Y|} p_{d}(y_{i}|x))}{(\sum_{i \not = j}^{|Y|} p_{\theta}(y_{i}|x))} \nonumber\\
     =& p_{d}(y_{j}|x)\log\frac{p_{d}(y_{j}|x)}{p_{\theta}(y_{j}|x)} + (1- p_{d}(y_{j}|x))\log\frac{(1 - p_{d}(y_{j}|x))}{(1 - p_{\theta}(y_{j}|x))}.
     \label{eq:div:sce:multi}
\end{align}
Eq. (\ref{eq:div:sce:multi}) shows that the divergence of multi labels is larger than that of binary labels in SCE.
As we explained, $d_{\Psi(\mathbf{z})}(f,g)$ of SCE is larger than $d_{\Psi(z)}(f,g)$ of NS in binary labels.
Therefore, the SCE loss imposes a larger penalty on the same predicted value than the NS loss when the value of the learning target is the same between the two losses.

\end{document}